\documentclass{article}
\usepackage{ijcai19}

\usepackage{times}
\usepackage{soul}
\usepackage{url}
\usepackage[hidelinks]{hyperref}
\usepackage[utf8]{inputenc}
\usepackage[small]{caption}
\usepackage{graphicx}
\usepackage{amsmath}
\usepackage{booktabs}
\usepackage{algorithm}
\usepackage{algorithmic}
\usepackage{bm}
\usepackage{amssymb,amsthm,threeparttable}
\usepackage{stfloats}
\usepackage{subfigure}

\newtheorem{lemma}{Lemma}
\title{Rectified Decision Trees: Towards Interpretability, \\
 Compression and Empirical Soundness}

\author{
Jiawang Bai$^1$\thanks{equal contribution.}
\and
Yiming Li$^{2\, *}$\and
Jiawei Li$^{1}$\and
Yong Jiang$^{2}$\And
Shutao Xia$^1$
\affiliations
$^1$Graduate School at Shenzhen, Tsinghua University, China\\
$^2$Tsinghua-Berkeley Shenzhen Institute, Tsinghua University, China
\emails
baijw1020@gmail.com,
\{li-ym18, li-jw15\}@mails.tsinghua.edu.cn,
\{jiangy, xiast\}@sz.tsinghua.edu.cn
}

\begin{document}

\maketitle
\begin{abstract}
How to obtain a model with good interpretability and performance has always been an important research topic. In this paper, we propose rectified decision trees (ReDT), a knowledge distillation based decision trees rectification with high interpretability, small model size, and empirical soundness. Specifically, we extend the impurity calculation and the \textit{pure} ending condition of the classical decision tree to propose a decision tree extension that allows the use of soft labels generated by a well-trained teacher model in training and prediction process. It is worth noting that for the acquisition of soft labels, we propose a new multiple cross-validation based method to reduce the effects of randomness and overfitting. These approaches ensure that ReDT retains excellent interpretability and even achieves fewer nodes than the decision tree in the aspect of compression while having relatively good performance. Besides, in contrast to traditional knowledge distillation, back propagation of the student model is not necessarily required in ReDT, which is an attempt of a new knowledge distillation approach.  Extensive experiments are conducted, which demonstrates the superiority of ReDT in interpretability, compression, and empirical soundness. 
\end{abstract}

\section{Introduction}
Random forests is a typical ensemble learning method, where a large number of randomized decision trees are constructed and the results from all trees are combined for the final prediction of the forest. Since its introduction in \cite{breiman2001}, random forests and its several variants \cite{friedman2001,chen2016} have been widely used in many fields, such as deep learning \cite{Zhou2017,feng2018} and even outlier detection \cite{liu2008}. In addition to its application, its theoretical properties have also been extensively studied. \cite{denil2014,scornet2015}.

However, although those complicated algorithms, such as random forests and GBDT, reach great success in many aspects, this high prediction performance makes considerable sacrifices of interpretability. The essential procedures of ensemble approaches cause this decline. For example, comparing to decision trees, the bootstrap and voting process of random forests makes the predictions much more difficult to explain. On the contrary, the decision trees are known to have the best interpretability among all machine learning algorithms yet with relatively lousy performance. Besides, forest-based algorithms or even deep neural networks (DNN) usually require much larger storage than decision trees, which is unacceptable especially when the model is set on a personal device with strict storage limitations (such as a cellular device). This conflict between empirical soundness and interpretability with flexible storage continuously drives researchers.

To address these problems, using a tree ensemble to generate additional samples for the further construction of the decision tree is proposed by Breiman \cite{breiman1996}, which can be regarded as the first attempt for this problem. In \cite{meinshausen2010}, Node Harvest is introduced to simplify tree ensemble by using the shallow parts of the trees. The shortcoming of Node Harvest is that the simplified derived model is still an ensemble and therefore the challenge of interpretation remains. Recently, a distillation-based method is proposed, where the soften labels are generated by well-trained DNN to create a more understandable model in the form of a soft decision trees \cite{frosst2017}. However, since this method relies on the backpropagation of the soft decision trees, it cannot be used in the classical decision trees. Besides, the interpretability of the soft decision trees \cite{irsoy2012} is much weaker than the classical decision trees. 

In this paper, we propose rectified decision trees (ReDT), a knowledge distillation based decision trees rectification with high interpretability, empirical soundness and even has a smaller model size compared to the decision trees. The critical difference between ReDT and decision tree lies in the use of softening labels, which is the weighted average of soft labels (the output probability vector of a well-trained teacher model) and hard labels, in the process of building trees. Specifically, to construct a decision tree, the hard label is mainly involved in the two parts of the tree construction process: (1) calculating the change of impurity and (2) determining whether the node is \emph{pure} in the stopping condition. In our method, we introduce soften labels into these processes. Firstly, we calculate the average of the soften labels in the node. The proportion of the samples with $i$-th category needed in the calculation of impurity criterion is re-determined using the value of the $i$-th dimension of the soften label. Secondly, since it is almost impossible for the mixed labels of all samples in a node to be the same, we propose to use a pseudo-category which is corresponding to the soften label of the sample. Then the original stopping condition can remain. In ReDT, the teacher model can be DNN or any other classification algorithm and therefore ReDT is universal. In contrast to traditional knowledge distillation, back propagation of the student model is not necessarily required in ReDT, which can be regarded as an attempt of a new knowledge distillation approach. Besides, we propose a new multiple cross-validation based method to reduce the effects of randomness and overfitting.

The main contributions of this paper can be stated as follows: 1) We propose a decision trees extension, which is the first tree that allows training and predicting using soften labels; 2) The first universal back propagation-free distillation framework is proposed and 3) the empirical analysis of its mechanism is conducted; 4) We propose a new soften labels acquisition method based on multiple cross-validations to reduce the effects of randomness and overfitting; 5) Extensive experiments demonstrate the superiority of our approach in interpretability, compression, and empirical soundness.

\section{Related Work}
The interpretability of complex machine learning models, especially ensemble approaches and deep learning, has been widely concerned. At present, the most widely used machine learning models are mainly forest-based algorithms and DNN, so their interpretability is of great significance. There are a few previous studies on the interpretability of forest-based algorithms. The first work is done by Breiman \cite{breiman1996}, who propose to use tree ensemble to generate additional samples for the further construction of a single decision tree. In \cite{meinshausen2010}, Node harvest is proposed to simplify tree ensembles by using the shallow parts of the trees. Considering the simplification of tree ensembles as a model selection problem, and using the Bayesian method for selection is also proposed in \cite{hara2018}. The interpretability research of DNN mainly on three aspects: visualizing the representations in intermediate layers of DNN \cite{zeiler2014,zhou2018}, representation diagnosis \cite{yosinski2014,zhang2018} and build explainable DNNs \cite{chen2016b,sabour2017}. Recently, a knowledge distillation based method is provided, which uses a trained DNN to create a more explainable model in the form of soft decision trees \cite{frosst2017}. 

The compression of forest-based algorithms and DNN has also received extensive attention. A series of work focuses on \emph{pruning} techniques for forest-based algorithms, whose idea is to reduce the size by removing redundant components while maintaining the predictive performance \cite{quinlan1993,ren2015,nan2016}. The idea of \emph{pruning} is also widely used in the compression of DNN \cite{han2015,he2017}. Recently, extensive researches have been conducted on compression methods based on coding or quantization. \cite{han2015a,painsky2016}.

Recently, Knowledge distillation has been widely accepted as a compression method. The concept of knowledge distillation in the teacher-student framework by introducing the teacher’s softened output is first proposed in \cite{hinton2015}. Since then, a series of improvements and applications of knowledge distillation have been proposed \cite{romero2015,yim2017}. At present, almost all knowledge distillation focus on the compression of DNN and require the back-propagation of the student model. Besides, using knowledge distillation to distill DNN into a soften decision tree to achieve great interpretability and compressibility is recently proposed in \cite{frosst2017}. This method can be regarded as the first attempt to apply knowledge distillation to interpretability.

\section{The Proposed Method}
We present the rectified decision trees (ReDT) in this section. The main concepts of our proposed method are how to define the important information of the teacher model (distilled knowledge) and how we use it in training the student model ($i.e.$ the ReDT). Section \ref{DK} introduces the distilled knowledge that we further used in the construction of ReDT. Section \ref{TC} and \ref{Pre} discuss the specific construction and prediction process of ReDT. An empirical analysis, which demonstrates why soften labels can reach better performance than hard labels in the construction of the decision tree, is provided in section \ref{EA}. 

\subsection{Distilled Knowledge}\label{DK}
Let $\mathcal{D}_n$ represents a data set consisting of $n$ $i.i.d.$ observations. Each observation has the form $(\bm{X},Y)$, where $\bm{X} \in \mathbb{R}^D$ represents the $D$-dimensional features and $Y \in \{1, \cdots, K\}$ is the corresponding label of the observation. The label of a sample can be regarded a single sampling from a $K$-dimensional discrete distribution. Let \textbf{hard label} $\bm{y_{hard}}$ denotes the one-hot representation of the label. ($K$-dimensional vector, where the value in the dimension corresponding to the category is 1 and the rest are all 0). 

From the perspective of probability, the training of the model can be considered as an approximation of the distribution of data. It is extremely difficult to recover the true distribution of $(X, Y)$ from the hard labels directly. In contrast, the output of a well-trained model consists of a significant amount of useful information compared to the original hard label itself. Inspired by this idea, we define the \textbf{soft label} $\bm{y_{soft}}$, which is the output probability vector of a well-trained model such as DNN, random forests and GBDT, as the distilled knowledge from the teacher model. This idea is also partly supported by \cite{hinton2015} where he used a softened version of the final output of a teacher network to teach information to a small student network.

Once a well-trained teacher model is given, the generation of the soft label is straightforward by directly outputting the probability of all training samples. However, the acquisition of teacher model is usually needed through training. The most straightforward idea is to train the teacher model using all training samples and output the soft label of those samples. However, the soft label obtained through this process has relatively poor quality due to the effects of randomness and overfitting. This problem does not exist in the previous knowledge distillation task since their training is carried out simultaneously rather than strictly one after the other, thanks for the teacher model and the student model can both be trained through back propagation. To address this problem, we propose a multiple cross-validation based methods to calculate soft labels. Specifically, if 5 times 5-fold cross validation is implemented, we first randomly divide the training set into five similarly sized sets, then using four sets of data for training, and the other set of data to predict ($i.e.$ output its soft label). In each time, each sample is predicted once so that each sample will end up with 5 soft labels. And the final soft label is the average of all its predictions.

\subsection{ReDT Construction}\label{TC}
In the proposed ReDT, comparing to the original decision tree, there are two main alterations including the calculation of impurity decrease and the stopping condition. In our method, we introduce soften label into these processes.

Note that instead of using the soft label of samples directly, we use the \textbf{mixed label} $\bm{y_{mixed}}$, which is the weighted average of soft label and hard label with weight hyperparameter $\alpha \in [0, 1]$. That is,
\begin{equation}\label{mixed}
    \bm{y_{mixed}} = \alpha \bm{y_{hard}} + (1-\alpha) \bm{y_{soft}}.
\end{equation}

The hyperparameter $\alpha$ plays a role in regulating the proportion of using the soft label. The larger $\alpha$, the smaller the proportion of the soft label in the mixed label. When $\alpha = 1$, the ReDT becomes Breiman's decision trees. The purpose of using mixed labels is to consider that the soft label may have a certain degree of error. By adjusting the hyperparameter $\alpha$, we can obtain the soften label with sufficient information and relative accuracy.

Recall that in the classification problem, the impurity decrease caused by splitting point $v$ is denoted by
\begin{equation}\label{1}
    I(v) = T(\mathcal{D}) - \frac{|\mathcal{D}_l|}{|\mathcal{D}|}T(\mathcal{D}_l)
    -\frac{|\mathcal{D}_r|}{|\mathcal{D}|}T(\mathcal{D}_r),
\end{equation}
where $\mathcal{D}_{l}, \mathcal{D}_{r}$ are two children sets generated by $\mathcal{D}$ splitting at $v$, $T(\cdot)$ is the impurity criterion ($e.g.$ Shannon entropy or Gini index). The first alteration in ReDT is the probability $p_i$, which implies the proportion of the samples with $i$-th category, used in calculating the impurity decrease of a splitting point. Specifically, since each sample uses a soften label instead of a hard label, we calculate the average of the soften labels of all the samples in the node and finally obtain a $K$-dimensional vector. At this time, $p_i$ is redetermined as the value of the $i$-th dimension of that vector. In other words, let $\bm{y_{mixed}^{(j)}}=(y_{j1},y_{j2},\cdots,y_{jK})$ denotes the mixed label of $j$-th training sample. $p_i$ of node $\mathcal{N}$ is calculated by 
\begin{equation}
    p_i=\frac{1}{|\mathcal{N}|}\sum_{\bm{y_{mixed}^{(j)}}\in \mathcal{N}} y_{ji},
\end{equation}
where $|\mathcal{N}|$ denotes the number of samples in leaf node $\mathcal{N}$.

The second alteration is how to define \emph{pure} in the stopping condition. In the training process of original decision trees, if all samples in a node have a single category, the node is considered to be pure. At this point, the stopping condition is reached, and this node is no longer to split. However, in the ReDT, it is almost impossible for the mixed labels of all samples to be the same. Therefore, we use the category corresponding to the maximum probability in the mixed label of the sample as its pseudo-category $y_{pseudo}$, $i.e.$,
\begin{equation}\label{pseudo}
    y_{pseudo}=\arg \max \bm{y_{mixed}},    
\end{equation}
and then determining whether to continue to split based on original stopping condition with it. 

\begin{algorithm}[ht]
   \caption{The training process of ReDT: $ReDT()$ }
   \label{alg-dt}
\begin{algorithmic}[1]
   \STATE {\bfseries Input:} Training set $\mathcal{D}=\left\{(\bm{X},\bm{y_{mixed}})\right\}$ calculated according to (\ref{mixed}) and minimum leaf size $k$.
   \STATE {\bfseries Output:} The rectified decision tree $T$.
   \STATE Calculate pseudo-category $y_{pseudo}$ of each sample in $\mathcal{D}$ by (\ref{pseudo}).
   \STATE Determine whether the node is pure based on whether each sample in $\mathcal{D}$ has the same pseudo-category.
        \IF{$|X|>k$ and the node is not pure}
        \STATE Calculate the impurity decrease vector $I$ according to equation (\ref{1}).
        \STATE Select the splitting point with maximum impurity reduction criterion. 
        \STATE The training set $\mathcal{D}$ correspondingly split into two child nodes, called $\mathcal{D}_{l}, \mathcal{D}_{r}$.
        \STATE $T.leftchild \leftarrow ReDT(\mathcal{D}_l, k)$
        \STATE $T.rightchild \leftarrow ReDT(\mathcal{D}_r, k)$
        \ENDIF
\STATE {\bfseries Return:} $T$.
\end{algorithmic}
\end{algorithm}

\subsection{Prediction}\label{Pre}
Once the ReDT has grown based on the mixed label as described above, the predictions for a newly given sample can be made as follows.

Suppose the unlabeled sample is $\bm{x}$ and the predicted label and predicted discrete probability distribution of that sample is $\hat{y}$ and $\bm{P}=(\hat{p_1}, \cdots, \hat{p_K})$ respectively.

According to a series of decisions, $\bm{x}$ will eventually fall into a leaf node, assuming that node is $V$. The predicted distribution of $\bm{x}$ is the average of the mixed labels of all training samples falling into the leaf nodes $V$, i.e.,

\begin{equation}
    \bm{P}=(\hat{p_1}, \cdots, \hat{p_K}) = \frac{1}{|V|} \sum_{\bm{y_{mixed}^{(i)}} \in V} \bm{y_{mixed}^{(i)}},
\end{equation}
where $|V|$ denotes the number of samples in leaf node $V$.

The predicted label of $\bm{x}$ is the one with biggest probability in $\bm{P}$:
\begin{equation}
    \hat{y}= \arg \max_i \hat{p_i}.
\end{equation}

\subsection{Empirical Analysis}\label{EA}
The reason why soften labels rather than hard labels should be used can be further demonstrated from the perspective of the calculation of impurity and distribution approximation. The specific analyses are as follows:

\begin{lemma}[Integer Partition Lemma]\label{l1}
Suppose there is an integer $N$, which is the sum of $K$ integers $n_1, \cdots, n_K$, i.e., 
$$
N=n_1+n_2,+\cdots+n_k.
$$
There are totally $C_{n+k-1}^{k-1}=\frac{(n+k-1)!}{(k-1)!n!}$ possible values for the ordered pair $(n_1, \cdots,n_K)$.
\end{lemma}
\begin{proof}
This problem is equivalent to picking $k-1$ locations randomly from $n+k-1$ locations. The result is trivial based on the basics of number theory. 
\end{proof}

Lemma \ref{l1} indicates that for a $K$-classification problem, if the node $\mathcal{N}$ contains $N$ samples, then the impurity of this node has at most $C_{n+k-1}^{k-1}$ possible values. In other words, compared to soften label, the use of hard label limits the precision of the impurity of the nodes. This limitation has a great adverse effect on the selection of the split point, especially when the number of samples is relatively small.

From another perspective, the improvement brought by soft labels is since it is tough to recover the distribution of $(X, Y)$ with hard labels directly, especially when the number of samples is relatively small. However, once the relatively correct soften a well-trained teacher model provides labels, a large amount of information of the distribution is contained in it. The use of this information about the distribution makes the decision surface offset towards the real position compared to when using the hard label.

\subsection{Comparision between DT, SDT and ReDT}
Although both soft decision trees (SDT) and ReDT are the extension of DT, there are many differences between them. In this section, we compare DT, SDT and ReDT from five aspects including (1) interpretability, (2) empirical soundness, (3) back-propagation needed, (4) soften label allowed and (5) compression, as shown in Table \ref{Comparison}. The method that satisfies the aspect is marked by \checkmark.

It is worth noting that interpretability, empirical soundness, and compression are relative. For example, the interpretability of SDT is stronger than DNN but is much weaker than DT and ReDT. Besides, since back propagation of the student model is not necessarily required in ReDT, this new knowledge distillation approach can be easily extended to other model and preserving running efficiency.

\begin{table}[ht]
\caption{Comparison between DT, SDT and ReDT.}
\vskip -0.1in
\label{Comparison}
\begin{center}
\begin{small}
\begin{sc}
\begin{tabular}{|l|c|c|c|}
\hline
            & DT & SDT & \upshape{ReDT} \\ \hline
\upshape{Interpretability}   & \checkmark   &   & \checkmark      \\ \hline
\upshape{Empirical soundness}   &  & \checkmark & \checkmark \\ \hline
\upshape{Back-propagation needed}  &   & \checkmark  &   \\  \hline
\upshape{Soften label allowed}  &   & \checkmark       & \checkmark      \\ \hline
\upshape{Compression (small model size)}  &  \checkmark  &  & \checkmark     \\ \hline
\end{tabular}
\end{sc}
\end{small}
\end{center}
\vskip -0.15in
\end{table}

\section{Experiments}
\subsection{Configuration}
For the DNN configurations, the experiments were conducted on the benchmark dataset MNIST \cite{lecun1998}. All networks were trained using Adam and an initial learning rate of 0.1. The learning rate was divided by 10 after epochs 20 and 40 (50 epochs in total). We examine a variety of DNN architectures including MLP, LeNet-5, and VGG-11, and use ReLU for activation function, cross-entropy for loss function. The MLP has two hidden layers, with 784 and 256 units respectively and dropout rate 0.5 for hidden layers. Besides, the temperature used in generating soft labels in DNN is set to 4 as suggested in \cite{hinton2015}.

\begin{table}[ht]
\caption{Datasets description.}
\vskip -0.1in
\label{table1}
\begin{center}
\begin{small}
\begin{sc}
\begin{tabular}{lccc}
\toprule
Data set & Class & Features & Instances   \\
\midrule
adult      & 2  & 14  & 48842 \\
crx        & 2  & 15  & 690   \\
EEG        & 2  & 15  & 14980 \\
bank       & 2  & 17  & 45211 \\
german     & 2  & 20  & 1000  \\
cmc        & 3  & 9   & 1473  \\
connect-4  & 3  & 42  & 67557 \\
land-cover & 9  & 147 & 675   \\
letter     & 26 & 15  & 20000 \\
isolet     & 26 & 617 & 7797  \\
\bottomrule
\end{tabular}
\end{sc}
\end{small}
\end{center}
\vskip -0.15in
\end{table}

All datasets involved in the evaluation of forest-based teacher are obtained from the UCI repository \cite{UCI}. Their information are listed in Table \ref{table1}. %Those data sets include low, moderate, and high dimensional attributes therefore sufficiently representative to demonstrate and evaluate the performance of different algorithms.
Besides, $70\%$ data is used for training and other $30\%$ is used for testing. Here we use random forests (RF) and GBDT as the teacher model. They are the representative of the bagging and boosting method in forest-based teacher respectively. We determine the value of $\alpha$ by grid search in a step of 0.1 in the range $[0,1]$ and the implement of GBDT is refer on scikit-learn platform \cite{pedregosa2011}. The number of trees contained in both random forest and GBDT is all set to 100. Besides, the performance of the decision tree trained with hard labels is also provided as a benchmark. Compared with the classical decision tree, since the soft decision tree is more like a tree-shaped neural network and with much weaker interpretability, it is not compared as a benchmark in experiments. The Gini index was used in RF, DT and ReDT as the impurity measure and minimum leaf size $k=5$ is set for both RF, GBDT, DT, and ReDT as suggested in \cite{breiman2001}.

Besides, 5 times 5-fold cross-validation is used to calculate the soft label of the training set and Wilcoxons signed-rank test \cite{demvsar2006} is carried out to test for difference between the results of the ReDT and those of decision trees at significance level 0.05. Compared with decision trees, ReDT with better performance (higher accuracy or fewer number of nodes) is indicated in boldface. Those that had a statistically significant difference from the decision tree are marked with "$\bullet$". Besides, we carried out the experiment 10 times to reduce the effect of randomness.

\begin{table*}[ht]
\caption{Comparison of test accuracy of different forests-based teacher model.}
\label{table2}
\begin{center}
\begin{threeparttable}
\vskip -0.1in
\begin{small}
\begin{sc}
\begin{tabular}{lccccccc}
\toprule
Dataset & RF & GBDT & DT & {\upshape ReDT}(RF) & {\upshape ReDT}(GBDT) & $\alpha^{*}$(RF) & $\alpha^{*}$(GBDT)\\
\midrule
adult      & 86.54\% & 86.53\% & 81.86\% & \textbf{86.18\%}$^{\bullet}$ & \textbf{86.16\%}$^{\bullet}$ & 0.01         & 0.06         \\
crx        & 86.14\% & 86.09\% & 80.51\% & \textbf{85.46\%}$^{\bullet}$ & \textbf{84.40\%}$^{\bullet}$ & 0.08       & 0.11           \\
EEG        & 81.50\% & 90.58\% & 82.88\% & \textbf{83.02\%} & \textbf{83.01\%} & 0.24       & 0.52         \\
bank       & 90.38\% & 90.41\% & 87.60\% & \textbf{90.11\%}$^{\bullet}$ & \textbf{90.15\%}$^{\bullet}$ & 0.06         & 0.03         \\
german     & 76.60\% & 76.13\% & 68.37\% & \textbf{73.40\%}$^{\bullet}$ & \textbf{72.67\%}$^{\bullet}$ & 0.07       & 0.10           \\
cmc        & 55.15\% & 55.66\% & 48.31\% & \textbf{55.05\%}$^{\bullet}$ & \textbf{55.41\%}$^{\bullet}$ & 0         & 0           \\
connect-4  & 75.38\% & 77.58\% & 71.73\% & \textbf{76.69\%}$^{\bullet}$ & \textbf{76.02\%}$^{\bullet}$ & 0.30       & 0.30         \\
land-cover & 83.69\% & 83.80\% & 76.55\% & \textbf{77.59\%} & \textbf{77.14\%} & 0.54       & 0.37         \\
letter     & 91.56\% & 93.61\% & 85.65\% & \textbf{86.01\%} & \textbf{86.15\%} & 0.9       & 0.9         \\
isolet     & 93.68\% & 93.32\% & 79.83\% & \textbf{81.40\%}$^{\bullet}$ & \textbf{81.77\%}$^{\bullet}$ & 0.57       & 0.33         \\
\bottomrule
\end{tabular}
\end{sc}
\end{small}
\begin{tablenotes}
 \item $^{\bullet}$: ReDT is better than decision trees at a level of significance 0.05.
 \item $\alpha^{*}$: The average of all best $\alpha$ for each experiment.
\end{tablenotes}
\vskip -0.1in
\end{threeparttable}
\end{center}
\end{table*}

\begin{table*}[ht]
\caption{Comparison of the number of nodes of forests-based teacher distillation.}
\label{table3}
\begin{center}
\begin{threeparttable}
\vskip -0.1in
\begin{small}
\begin{sc}
\begin{tabular}{lccccc}
\toprule
Dataset    & RF     & GBDT  & DT    & {\upshape ReDT}(RF)      & {\upshape ReDT}(GBDT)    \\
\midrule
adult      & 244832 & 1486  & 7869  & \textbf{2286}$^{\bullet}$ & \textbf{2023}$^{\bullet}$ \\
crx        & 6191   & 1336  & 103   & \textbf{48}$^{\bullet}$   & \textbf{65}$^{\bullet}$   \\
EEG        & 125289 & 1489  & 1858  & 1948          & 1939          \\
bank       & 223906 & 1470  & 4302  & \textbf{1678}$^{\bullet}$ & \textbf{1603}$^{\bullet}$ \\
german     & 11063  & 1404  & 227   & \textbf{140}$^{\bullet}$  & \textbf{172}$^{\bullet}$  \\
cmc        & 16220  & 4206  & 630   & \textbf{202}$^{\bullet}$  & \textbf{275}$^{\bullet}$  \\
connect-4  & 470261 & 4426  & 18152 & \textbf{8813}$^{\bullet}$ & \textbf{8740}$^{\bullet}$ \\
land-cover & 6100   & 9492  & 85    & \textbf{43}$^{\bullet}$   & \textbf{49}$^{\bullet}$   \\
letter     & 168101 & 38631 & 2752  & \textbf{2464}$^{\bullet}$ & \textbf{2459}$^{\bullet}$ \\
isolet     & 58905  & 32751 & 707   & \textbf{464}$^{\bullet}$  & \textbf{593}$^{\bullet}$  \\
\bottomrule
\end{tabular}
\end{sc}
\end{small}
\begin{tablenotes}
 \item $^{\bullet}$: ReDT is better than decision trees at a level of significance 0.05.
\end{tablenotes}
\end{threeparttable}
\end{center}
\vskip -0.15in
\end{table*}

\subsection{DNNs Teacher}
We discuss the performance including test accuracy (ACC) and the number of nodes (NODE) of ReDT under different teacher models and compare it with DT and its teacher model in this section.

\begin{table}[ht]
\caption{Comparison on MNIST.}
\vskip -0.1in
\label{mnist}
\begin{center}
\begin{small}
\begin{sc}
\begin{tabular}{|l|c|c|c|}
\hline
            & MLP & {\upshape LeNet-5} & VGG-11 \\ \hline
ACC (DNN)   & 98.33\%   & 99.42\%       & 99.49\%      \\ \hline
ACC (DT)    & \multicolumn{3}{c|}{87.55\%} \\ \hline
NODE (DT)   & \multicolumn{3}{c|}{5957} \\ \hline
ACC ({\upshape ReDT})   & \textbf{88.21\%} & \textbf{88.57\%}$^{\bullet}$ & \textbf{88.53\%}$^{\bullet}$ \\ \hline
Node ({\upshape ReDT})  & \textbf{5361}$^{\bullet}$    & \textbf{5173}$^{\bullet}$    & \textbf{5803}$^{\bullet}$    \\  
\hline
\end{tabular}
\end{sc}
\end{small}
\begin{tablenotes}
 \item $^{\bullet}$: ReDT is better than DT at a level of significance 0.05.
\end{tablenotes}
\end{center}
\vskip -0.1in
\end{table}

As shown in Table \ref{mnist}, although there is still a gap in ACC between ReDT and its teacher model since decision tree cannot learn the spatial relationships among the raw pixels, the ReDT have a remarkable improvement comparing to the original decision tree. Not to mention that in terms of compression, ReDT even has fewer nodes than DT (and therefore has a smaller model size).

\subsection{Forests-based Teacher}
Table \ref{table2} and \ref{table3} shows the test accuracy of different forest-based teacher model and the number of nodes respectively. Regardless of which teacher model is used, the ReDT has a remarkable improvement in both efficiency and compression. Among all ten data sets, ReDT has higher test accuracy than the decision tree, and this improvement is significant on seven of those data sets. Specifically, ReDT has achieved an increase of almost 5\% accuracy compared to DT on half of the data sets. In particular, in the three data sets (Band, ADULT, and CONNECT-4), ReDT has similar performance to its teacher model. Also, the value of optimal $\alpha$ seems to have some direct connection with the number of categories in the dataset. Specifically, data sets with more categories (such as LAND-COVER, ISOLET, and LETTER) generally have a larger optimal $\alpha$. In other words, a large proportion of hard label needs to be included in the mixed label to give ReDT excellent performance. Two reasons may cause this:  1) The more categories, the more likely the soft labels will contain more error information; 2) The more categories, the higher the interference caused by error information contained in soft labels. Regardless of the reason, the number of categories of samples can be used to provide the initial intuition of $\alpha$. In terms of compression, in nine of ten data sets, the number of nodes in ReDT is smaller than the decision tree. In other words, ReDT has a smaller model size than the decision tree, not to mention the teacher models, such as random forests and GBDT, which is usually more complicated. Overall, ReDT with the forest-based teacher has achieved a significant improvement in both performance and compression.

\begin{figure*}[ht]
\centering
\includegraphics[scale=0.57]{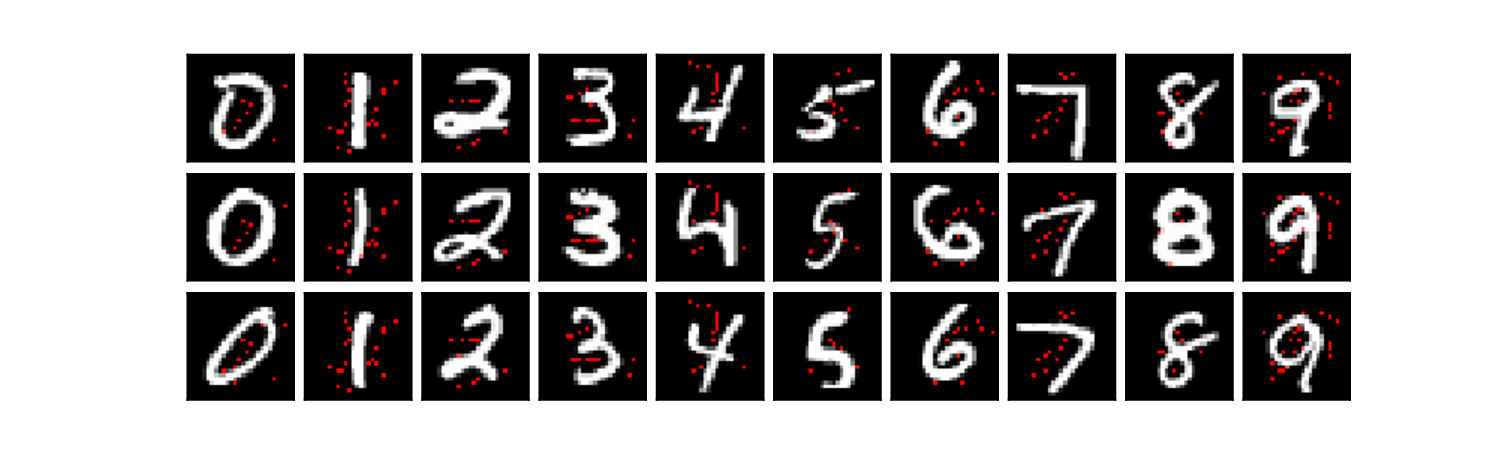}
\vskip -0.2in
\caption{Visualization of key pixels for MNIST image classification. (The key pixels are marked in red.)}
\label{interpretability}
\vskip -0.1in
\end{figure*}

\subsection{Discussion}
In this section, We discuss the compression, interpretability, and the impact of hyperparameters $\alpha$ on the model.

\subsubsection{Compression}
As we mentioned above, ReDT is an extension of a decision tree, and therefore the size of its model can still be measured by the number of nodes. Without loss of generality, we compare ReDT with Decision tree here. There are two advantages for such comparison: 1) The decision tree is almost the model requires the fewest number of nodes in the forest-based algorithms, not to mention its size is much smaller than the DNN or other complex algorithms. If the model has a relatively smaller size compared to the decision tree, then it must have excellent compression; 2) The size of the decision tree and ReDT model are both reflected by the number of nodes, which is convenient for comparison.

Without loss of generality, we use random forest and GBDT as the teacher model here. The compression rate $(\, \rm{Node (ReDT)}/\rm{Node(DT)}\, )$ of multiple data sets under different hyperparameter $\alpha$ as shown in Fig. \ref{Compression}. The smaller the compression rate, the smaller the model size of ReDT. 

It can be seen that the compression rate of almost every dataset is less than 1, which indicates that for all hyperparameter $\alpha$, ReDT has a smaller model size than DT in most cases. In addition, as the hyperparameter $\alpha$ increases, the compression rate has an upward trend. This is caused by the fact that the soft label carries a large amount of information about the distribution, whether it is correct or not, thus facilitating the decision tree to divide the data. The smaller the $\alpha$, the more significant the proportion of the soft label in the mixed label, therefore the smaller the size of the model. Thus, although there is no $\alpha$ such that it can correspond to the highest test accuracy on all datasets (because this is closely related to complex factors such as the correctness of the soft label, dataset, etc.), using $\alpha$ to adjust the size of the model is a good choice. Besides, as shown in the figure, the growth of the compression rate of the data set with more categories is significantly slower. Regardless of the reason, this opposite tendency between compression rate and accuracy (the more the categories, the larger the $\alpha^{*}$) allows ReDT to have a smaller model size when achieving empirical soundness.

\begin{figure}[!htb]
\vskip -0.15in
\centering
\subfigure[]{
\label{figa}
\includegraphics[width=4cm]{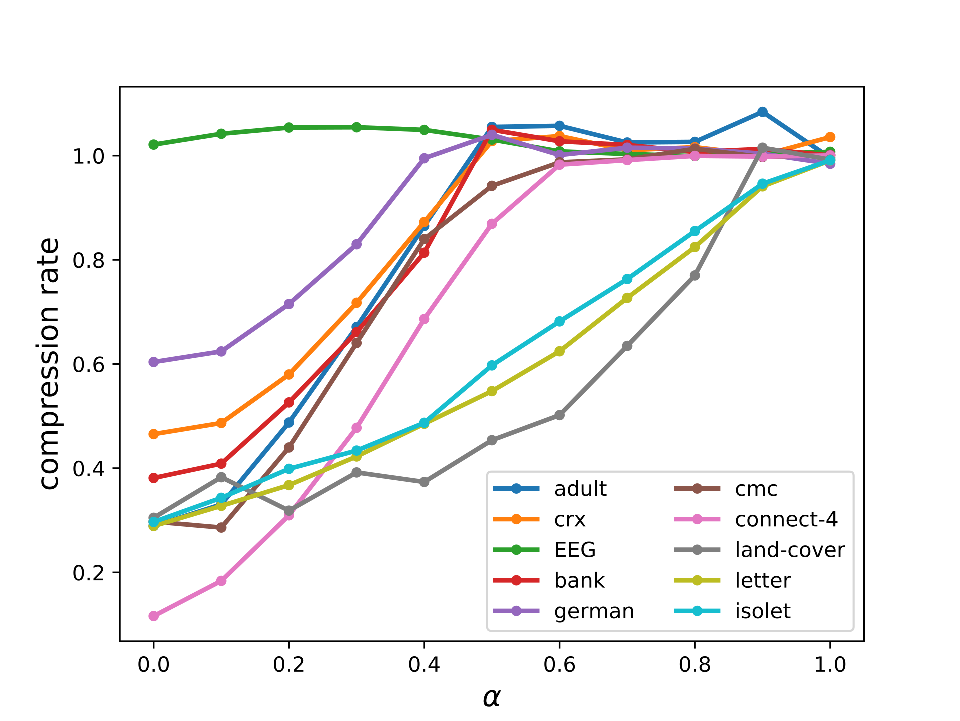}}
\subfigure[]{
\label{figb} %% label for second subfigure
\includegraphics[width=4cm]{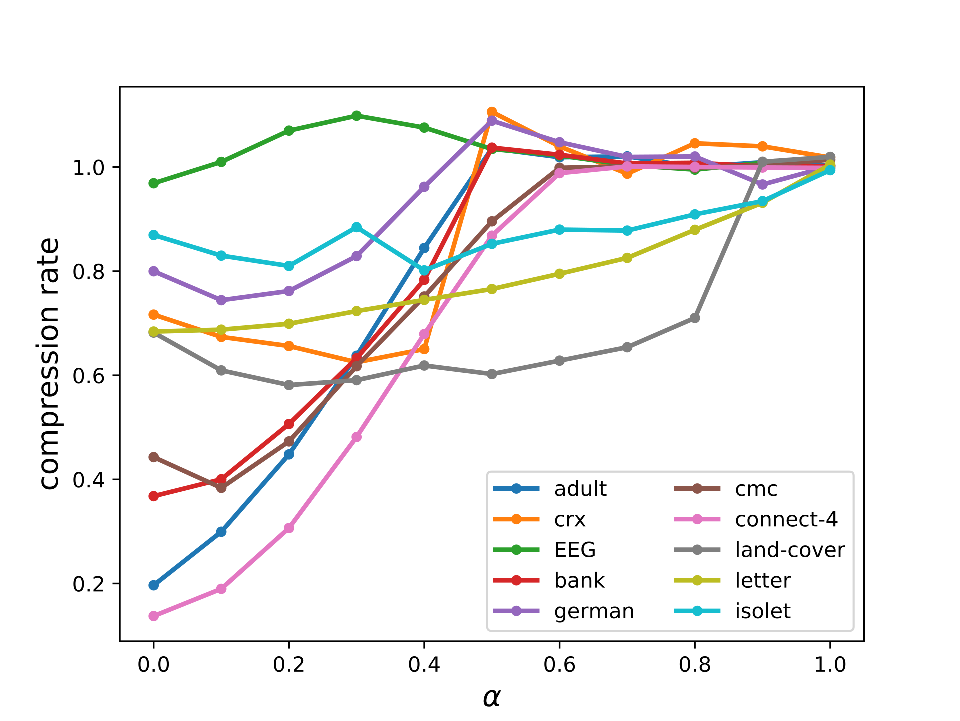}}
\vskip -0.15in
\caption{Compression rate under different teacher model. (a) Random forests teacher; (b) GBDT teacher.}
\label{Compression}
\vskip -0.1in
\end{figure}

\subsubsection{Interpretability}
The decision tree makes the prediction depending on the leaf node to which the input $\bm{x}$ belongs. The corresponding leaf node is determined by traversing the tree from the root. Although its path can represent the decision of the model, when the sample is in high dimensions, especially when it is a picture or speech, a single category will have a large number of different decision paths, and therefore it is difficult to explain the output by simply listing its path. To address this problem, we propose to highlight the key features (pixels) used in the sample's decision path.

Here, we use MNIST as an example to demonstrate the powerful interpretability of ReDT. We randomly select three samples for each number to predict. The pixels contained in its decision path, the key pixels, are marked in red, as shown in Fig. \ref{interpretability}. Although we don't have the ground true decision path, since the key pixel is almost the outline of the number, so the prediction is with highly interpretability and confidence.

\section{Conclusion}
By recognizing that the key of learning process lies in the approximation of data distribution, in this paper, we attempt to endow the great approximation ability of other teacher models to decision tree inspired by knowledge distillation and propose the ReDT method. Experiments and comparisons demonstrate that the ReDT remarkably surpasses original decision tree, and its performance is relatively competitive to its teacher model. More importantly, while having good performance, ReDT retains the excellent interpretability of the decision tree and even achieves smaller model size than the decision tree. Besides, in contrast to traditional knowledge distillation, back propagation of the student model is not necessarily required in ReDT, which can be regarded as an attempt of a new knowledge distillation approach. This new knowledge distillation method can be easily extended to other models.

%\section*{Acknowledgments}
%This research is supported in part by the Major State Basic Research Development Program of China (973 Program, 2012CB315803), the National Natural Science Foundation of China (61371078), and the Research Fund for the Doctoral Program of Higher Education of China (20130002110051).

\clearpage

\bibliographystyle{named}
\bibliography{ijcai19}

\end{document}